\documentclass[accepted]{uai2026} 
                        
\usepackage[british]{babel}

\usepackage{natbib} 
\usepackage{mathtools} 
\usepackage{amssymb}
\usepackage{float}
\usepackage{array}
\usepackage{booktabs} 
\usepackage{tikz} 
\usepackage[ruled,lined]{algorithm2e}
\usepackage{bm}
\usepackage{amsthm}
\usepackage{makecell}
\usepackage{multirow}
\newcolumntype{M}[1]{>{\centering\arraybackslash}m{#1}}
\newtheorem{theorem}{Theorem}
\newtheorem{lemma}[theorem]{Lemma}
\newtheorem{proposition}[theorem]{Proposition}
\newtheorem{definition}[theorem]{Definition}
\newtheorem{assumption}{Assumption}
\newtheorem{property}[theorem]{Property}
\newtheorem{remark}[theorem]{Remark}

\title{Equivariant Neural Belief Propagation}

\author[1]{Zehua Cheng}
\author[2]{Wei Dai}
\author[2]{Jiahao Sun}
\affil[1]{%
    Department of Computer Science\\
    University of Oxford\\
    Oxford, United Kingdom
}
\affil[2]{%
    FLock.io\\
    London, United Kingdom
}
\begin{document}
\maketitle

\begin{abstract}
Probabilistic inference over spatially embedded variables requires beliefs that respect $SE(3)$ symmetry, yet existing equivariant networks produce only scalars and vectors---not the rank-2 precision tensors needed for anisotropic uncertainty---and single-component messages collapse multi-modal energy landscapes to physically meaningless averages. We introduce Equivariant Neural Belief Propagation (ENBP), a factor-graph framework whose messages are equivariant Gaussian mixture models with sufficient statistics that transform exactly under $SE(3)$. Rank-2 precision matrices are synthesised via equivariant outer products, ingested through differentiable spectral decomposition, and kept tractable by a greedy KL-based mixture reduction that provably commutes with $SE(3)$. On GEOM-QM9 and GEOM-Drugs, ENBP achieves 98.9\% conformational coverage at 0.090~$\mathring{A}$ error with sub-second latency---over $100\times$ faster than diffusion baselines at higher accuracy. On multi-body robotic inference, vanilla loopy BP diverges at 15+ agents while ENBP converges with near-zero collision rates and machine-precision equivariance error (${\sim}10^{-7}$ vs.\ $10^{-1}$ for augmented baselines).
\end{abstract}

\section{Introduction}

Probabilistic inference over spatially embedded variables---predicting atomic positions in a molecule or coordinate marginals of robotic agents---is governed by $SE(3)$-invariant energy landscapes. Predicted beliefs must therefore transform equivariantly with the input \citep{satorras2021en, batzner2022e3}. Belief Propagation (BP) on factor graphs provides the natural scaffold, computing exact marginals on trees \citep{pearl1988probabilistic, kschischang2001factor} with principled extensions to loopy topologies \citep{yedidia2003understanding}. The question is not whether to use equivariant architectures or factor-graph inference, but how to unite them---a unification that existing methods have failed to achieve for structural reasons.

The root cause is a forced dilemma between geometric symmetry and representational expressivity. Current $E(n)$-equivariant networks \citep{satorras2021en, thomas2018tensor} produce invariant scalars and equivariant vectors, which suffice for energies or forces but not for probabilistic inference: the covariance of a distribution in $\mathbb{R}^3$ is a rank-2 tensor with transformation law $\boldsymbol{\Lambda} \mapsto R\boldsymbol{\Lambda}R^\top$, which is neither scalar nor vectorial. Preserving equivariance forces isotropic covariances ($\boldsymbol{\Lambda} = \lambda \mathbf{I}$), unable to distinguish radial from angular precision \citep{ganea2021geomol, schutt2017schnet}; breaking equivariance introduces systematic generalisation error that no finite augmentation eliminates \citep{elesedy2021provably}. A second deficiency compounds the first: single-Gaussian messages collapse multi-modal energy landscapes to the mean of distinct modes---a location that may carry negligible physical probability. Together, the rank-2 bottleneck and uni-modal constraint leave existing equivariant frameworks \emph{unable to represent} the geometric and statistical complexity of physical distributions.

We propose \emph{Equivariant Neural Belief Propagation} (ENBP), resolving both pathologies through a single principle: messages are \emph{equivariant Gaussian mixture models} (EGMMs) whose sufficient statistics transform exactly under $SE(3)$. Three interlocking mechanisms realise this principle:
\begin{enumerate}
    \item \textbf{Precision tensor synthesis.} Each component's precision matrix is built from learned equivariant basis vectors via $\boldsymbol{\Lambda}_k = \sum_p \sigma_k^{(p)} \mathbf{v}_k^{(p)} {\mathbf{v}_k^{(p)}}^\top + \epsilon \mathbf{I}$, guaranteeing positive-definiteness and $\boldsymbol{\Lambda}_k \mapsto R\boldsymbol{\Lambda}_k R^\top$.
    \item \textbf{Spectral ingestion.} Incoming precision matrices are eigendecomposed into invariant eigenvalues and equivariant eigenvectors, enabling standard equivariant layers to consume full anisotropic covariance.
    \item \textbf{Equivariant mixture reduction.} The exponential blowup from multiplying $K$-component mixtures is controlled by greedy KL-based merging that provably commutes with $SE(3)$.
\end{enumerate}

Experiments are designed to falsify the diagnosis. Isotropic precision causes large, systematic degradation; scaling $K{=}1 \to 4$ triggers a phase transition in collision avoidance on loopy robotic graphs; and ENBP achieves equivariance error ${\sim}10^{-7}$ while augmented baselines remain at ${\sim}10^{-1}$. On GEOM-QM9 and GEOM-Drugs, ENBP reaches 98.9\% coverage at 0.090~$\mathring{A}$ error with sub-second latency---over $100\times$ faster than diffusion baselines at higher accuracy. On multi-body robotic inference, vanilla loopy BP diverges at 15+ agents while ENBP converges with near-zero collision rates.

In summary, the contributions of this paper are:
\begin{enumerate}
    \item An equivariant Gaussian mixture message representation that simultaneously captures multi-modal distributions and full anisotropic covariance while transforming exactly under $SE(3)$.
    \item A rank-2 tensor synthesis--ingestion loop---outer-product precision construction paired with differentiable spectral decomposition---that enables standard equivariant architectures to produce and consume covariance matrices without violating symmetry.
    \item Formal proofs of positive-definiteness (Lemma~1), $SO(3)$-equivariance of precision tensors (Lemma~2), equivariance-preserving spectral ingestion (Proposition~1) and mixture reduction (Proposition~2), loss invariance (Theorem~2), end-to-end $SE(3)$-equivariance (Theorem~3), and convergence stabilisation (Lemma~3).
    \item State-of-the-art results on molecular conformation and multi-body robotic inference, with ablations confirming that each architectural component is individually necessary.
\end{enumerate}

\begin{figure*}\centering
    \includegraphics[width=0.95\textwidth]{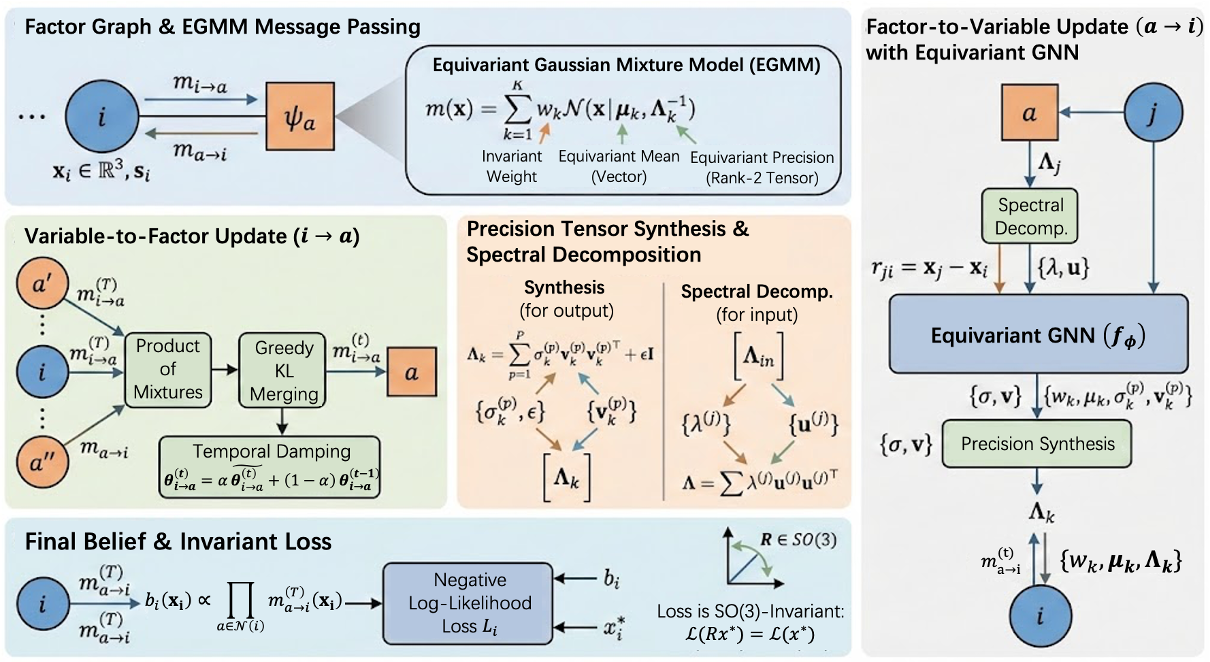}
    \caption{Overview of the Equivariant Neural Belief Propagation (ENBP) architecture. (1) Factor Graph \& EGMM Message Passing: The probabilistic inference problem is structured as a bipartite factor graph where variable and factor nodes exchange messages. To capture multi-modal energy landscapes, messages are parameterized as Equivariant Gaussian Mixture Models (EGMMs). Each mixture component consists of an invariant mixing weight, an $SE(3)$-equivariant mean vector, and an $SO(3)$-equivariant rank-2 precision matrix.}
\end{figure*}

\section{Related Works}
\label{sec:related_works}

\paragraph{Belief propagation and probabilistic inference on graphs.}
Belief propagation (BP) was introduced by \citet{pearl1988probabilistic} for exact inference on tree-structured graphical models and later extended to general factor graphs through the sum-product formalism \citep{kschischang2001factor}. When applied to graphs containing cycles, loopy BP can oscillate or diverge \citep{murphy1999loopy}, and its fixed points correspond to stationary points of the Bethe free energy rather than true marginals \citep{yedidia2003understanding}. Various strategies have been proposed to stabilise loopy BP, including message damping \citep{heskes2004uniqueness}, tree-reweighted methods \citep{wainwright2003tree}, and convergent alternatives based on convex relaxations \citep{globerson2007fixing}. In continuous domains, Expectation Propagation \citep{minka2001expectation} approximates intractable messages by moment-matching within the exponential family, while Variational Message Passing \citep{winn2005variational} restricts the variational distribution to a factorised form. Our framework inherits the iterative message-passing structure of BP but replaces scalar or single-Gaussian messages with equivariant Gaussian mixtures, operating directly on the factor graph topology rather than collapsing it.

\paragraph{Geometric and equivariant neural networks.}
Encoding physical symmetries directly into neural architectures has yielded substantial gains in data efficiency and generalisation \citep{elesedy2021provably}. Tensor Field Networks \citep{thomas2018tensor} pioneered the use of spherical harmonics to construct $SE(3)$-equivariant layers, and subsequent architectures have simplified this construction through invariant message passing \citep{schutt2017schnet, schutt2021equivariant}, equivariant coordinate updates \citep{satorras2021en}, or steerable features \citep{brandstetter2022geometric, batzner2022e3}. PaiNN \citep{schutt2021equivariant} interleaves equivariant message and update blocks to process both scalar and vectorial features, while DimeNet \citep{gasteiger2020directional} incorporates directional information through Bessel basis embeddings of interatomic angles. SEGNN \citep{brandstetter2022geometric} generalises steerable convolutions to arbitrary tensor orders on graphs. These methods, however, are designed for discriminative prediction (e.g.\ energy or force regression) and output point estimates rather than full probability distributions. ENBP addresses this gap by embedding an $E(n)$-equivariant network \citep{satorras2021en} within a belief propagation loop, enabling it to produce calibrated, multi-modal marginal distributions while retaining exact geometric equivariance.

\paragraph{Molecular conformation generation.}
Predicting three-dimensional molecular conformations from a connectivity graph has received growing attention as a testbed for geometric generative models. Early learning-based approaches predicted interatomic distances or torsion angles directly \citep{ganea2021geomol, xu2022confgnn}, but struggled to capture multi-modal torsional distributions. Score-based diffusion models brought substantial improvements: GeoDiff \citep{xu2022geodiff} applies a diffusion process in Cartesian coordinate space, while Torsional Diffusion \citep{jing2022torsional} restricts the diffusion to the hypertorus of dihedral angles, significantly reducing the dimensionality of the generative problem. More recent methods explore flow matching on Riemannian manifolds \citep{klein2024equivariant} and latent-space diffusion with $SE(3)$ invariance. Despite their strong empirical performance, diffusion-based methods require many sequential denoising steps at inference time, incurring high computational cost. In contrast, ENBP generates a full multi-modal belief in a fixed number of BP iterations, yielding competitive coverage and structural accuracy at a fraction of the latency.


\section{Methodology}
\label{sec:methodology}

Formal definitions of factor graphs, belief propagation, and group equivariance are provided in Appendix~\ref{app:preliminaries}.

\subsection{Problem Formulation and Factor Graph Structure}

We consider probabilistic inference over $N$ continuous spatial variables $\mathbf{x} = \{{\mathbf{x}}_1, \dots, {\mathbf{x}}_N\}$, where each ${\mathbf{x}}_i \in \mathbb{R}^3$ represents the position of a physical entity (e.g.\ an atom or robotic agent). Each variable node carries invariant scalar features $s_i \in \mathbb{R}^{d_s}$ (e.g.\ atomic number) that do not transform under spatial operations. The joint distribution factorises over $M$ local potentials \citep{pearl1988probabilistic, kschischang2001factor}:
\begin{equation}
\label{eq:joint}
p(\mathbf{x}) = \frac{1}{Z} \prod_{a=1}^{M} \psi_a(\mathbf{x}_{\mathcal{N}(a)}),
\end{equation}
where $\mathcal{N}(a) \subset \{1,\dots,N\}$ indexes the variables in factor $a$ and $Z$ is the partition function. In physical systems, $p$ is $SE(3)$-invariant: $p(R\mathbf{x} + \mathbf{t}) = p(\mathbf{x})$ for all $R \in SO(3)$, $\mathbf{t} \in \mathbb{R}^3$. Any inference framework must therefore produce equivariant marginal beliefs.

We represent Eq.~\eqref{eq:joint} as a bipartite factor graph $\mathcal{G} = (\mathcal{V}, \mathcal{F}, \mathcal{E})$, preserving the bipartite topology to encode conditional independence and accommodate higher-order interactions within single factor nodes. On acyclic graphs, BP computes exact marginals \citep{pearl1988probabilistic}; on the cyclic graphs arising in closed-ring molecules and coupled robotic formations, standard BP may diverge \citep{yedidia2003understanding, murphy1999loopy}, motivating the stabilisation mechanisms in Section~\ref{subsec:message_updates}.

\subsection{Equivariant Gaussian Mixture Messages}
\label{subsec:gmm_messages}

Messages should represent probability distributions rather than opaque latent vectors \citep{sudderth2003nonparametric, winn2005variational}. Single-Gaussian messages collapse multi-modal landscapes to a mean that may carry negligible physical probability \citep{ganea2021geomol}. We therefore parameterise each message as an Equivariant Gaussian Mixture Model (EGMM) with $K$ components:
\begin{equation}
\label{eq:egmm}
m_{a \to i}(\mathbf{x}_i) = \sum_{k=1}^{K} w_k \, \mathcal{N}\!\left(\mathbf{x}_i \,\middle|\, \boldsymbol{\mu}_k,\, \boldsymbol{\Lambda}_k^{-1}\right),
\end{equation}
where $w_k > 0$ with $\sum_k w_k = 1$, $\boldsymbol{\mu}_k \in \mathbb{R}^3$ is an $SE(3)$-equivariant mean, and $\boldsymbol{\Lambda}_k \in \mathbb{R}^{3 \times 3}$ is an $SO(3)$-equivariant precision matrix. Under rotation $R$, the weights are invariant, $\boldsymbol{\mu}_k \mapsto R\boldsymbol{\mu}_k$, and $\boldsymbol{\Lambda}_k \mapsto R\boldsymbol{\Lambda}_k R^\top$.

\subsection{Precision Tensor Synthesis via Equivariant Outer Products}
\label{subsec:precision}

Standard $E(n)$-equivariant networks \citep{satorras2021en} produce scalars and vectors but cannot output rank-2 tensors; scalar-only construction restricts $\boldsymbol{\Lambda}_k$ to the isotropic form $\lambda_k \mathbf{I}$, which cannot capture the anisotropic uncertainties of molecular bonds \citep{batzner2022e3, schutt2017schnet}. We synthesise $\boldsymbol{\Lambda}_k$ from $P$ learned equivariant basis vectors $\{\mathbf{v}_k^{(p)}\}$ and positive scalar coefficients $\{\sigma_k^{(p)}\}$ (enforced via softplus):
\begin{equation}
\label{eq:precision}
\boldsymbol{\Lambda}_k = \sum_{p=1}^{P} \sigma_k^{(p)} \, \mathbf{v}_k^{(p)} {\mathbf{v}_k^{(p)}}^\top + \epsilon \, \mathbf{I},
\end{equation}
where $\epsilon > 0$ guarantees positive-definiteness. Since each $\mathbf{v}_k^{(p)} \mapsto R\mathbf{v}_k^{(p)}$, the full construction satisfies $\boldsymbol{\Lambda}_k \mapsto R\boldsymbol{\Lambda}_k R^\top$. Setting $P \geq 3$ yields full rank when the basis vectors are not co-planar; in all experiments $P = 3$ and $\epsilon = 10^{-4}$.

\subsection{Spectral Decomposition for Tensor Ingestion}
\label{subsec:spectral}

The receiving node must ingest the rank-2 precision tensor as input features. Standard equivariant layers \citep{satorras2021en, thomas2018tensor} accept scalars and vectors but not $3 \times 3$ matrices; flattening would destroy the $SO(3)$ transformation law. We apply a differentiable eigendecomposition:
\begin{equation}
\label{eq:eigen}
\boldsymbol{\Lambda}_k = \sum_{j=1}^{3} \lambda_k^{(j)} \, \mathbf{u}_k^{(j)} {\mathbf{u}_k^{(j)}}^\top,
\end{equation}
where $\lambda_k^{(1)} \geq \lambda_k^{(2)} \geq \lambda_k^{(3)} > 0$ and $\{\mathbf{u}_k^{(j)}\}$ form an orthogonal eigenbasis. Under rotation, eigenvalues are invariant and eigenvectors co-rotate ($\mathbf{u}_k^{(j)} \mapsto R\mathbf{u}_k^{(j)}$), so the eigenvalues enter as scalar features and the eigenvectors as vector features.

Two numerical subtleties require treatment. Near-degenerate eigenvalues produce singular gradients via $(\lambda_k^{(i)} - \lambda_k^{(j)})^{-1}$ terms; following \citet{wang2021robust}, we replace each denominator with a stable approximant regularised by $\delta = 10^{-6}$. The sign ambiguity of eigenvectors ($\mathbf{u} \leftrightarrow -\mathbf{u}$) is resolved by enforcing that the component with the largest absolute value is positive.

\subsection{Message Update Rules}
\label{subsec:message_updates}

The inference algorithm proceeds through $T$ iterations of alternating variable-to-factor and factor-to-variable updates (Algorithm~\ref{alg:enbp}).

\paragraph{Variable-to-factor update.} Variable $i$ multiplies all incoming messages from $\mathcal{N}(i) \setminus \{a\}$ \citep{kschischang2001factor}. For $K$-component EGMMs, the exact product yields $K^{|\mathcal{N}(i)|-1}$ components. We handle this in two stages. First, pairwise products are computed via natural parameter summation:
\begin{equation}
\label{eq:nat_param}
\begin{aligned}
    &\boldsymbol{\Lambda}_{12} = \boldsymbol{\Lambda}_1 + \boldsymbol{\Lambda}_2,\\
    &\boldsymbol{\mu}_{12} = \boldsymbol{\Lambda}_{12}^{-1}(\boldsymbol{\Lambda}_1 \boldsymbol{\mu}_1 + \boldsymbol{\Lambda}_2 \boldsymbol{\mu}_2).
\end{aligned}
\end{equation}
Second, the expanded mixture is collapsed back to $K$ components by greedily merging the pair with the smallest closed-form KL divergence upper bound \citep{runnalls2007kullback}, using moment-matching. This reduction is equivariance-preserving: the KL criterion depends only on invariant quantities, and moment-matching of means commutes with rotation.

To mitigate the oscillatory behaviour of BP on loopy factor graphs \citep{heskes2004uniqueness}, we apply temporal damping. Let $\boldsymbol{\theta}_{i \to a}^{(t)}$ denote the concatenated natural parameters of the variable-to-factor message at iteration $t$. The damped update is:
\begin{equation}
\label{eq:damping}
\boldsymbol{\theta}_{i \to a}^{(t)} = \alpha \, \tilde{\boldsymbol{\theta}}_{i \to a}^{(t)} + (1 - \alpha) \, \boldsymbol{\theta}_{i \to a}^{(t-1)},
\end{equation}
where $\tilde{\boldsymbol{\theta}}_{i \to a}^{(t)}$ are the freshly computed parameters and $\alpha \in (0, 1]$ is a damping coefficient. This convex combination restricts the step size of the marginal updates, suppressing circular information flow in loopy substructures. While damping does not provide formal convergence guarantees for loopy BP in general \citep{heskes2004uniqueness, yedidia2003understanding}, it empirically stabilises the inference process in all configurations evaluated in Section~\ref{sec:experiments}.

\paragraph{Factor-to-variable update.} The factor-to-variable message marginalises the potential over all neighbouring variables except $i$:
\begin{equation}
\label{eq:ftov}
m_{a \to i}(\mathbf{x}_i) \propto \int \psi_a(\mathbf{x}_{\mathcal{N}(a)}) \prod_{j \in \mathcal{N}(a) \setminus \{i\}} m_{j \to a}(\mathbf{x}_j) \, d\mathbf{x}_j.
\end{equation}
This integral is generally intractable \citep{sudderth2003nonparametric, minka2001expectation}. We approximate it with a learnable $E(n)$-equivariant network $f_\phi$ \citep{satorras2021en} operating over factor $a$'s local neighbourhood. The network receives spectral-decomposed sufficient statistics (eigenvalues, eigenvectors, means, weights) from each $j \in \mathcal{N}(a) \setminus \{i\}$, and operates on relative positions $\mathbf{x}_j - \mathbf{x}_i$ to enforce translation invariance (Section~\ref{subsec:data_integrity}). It outputs basis vectors, scaling coefficients, means, and logit weights, from which EGMM parameters are assembled via Eq.~\eqref{eq:precision} and softmax normalisation. The equivariance of $f_\phi$ ensures the approximate marginalisation commutes with $SE(3)$ \citep{elesedy2021provably}.

\begin{algorithm}[t]
\caption{Equivariant Neural Belief Propagation}
\label{alg:enbp}
\SetKwInOut{Input}{Input}
\SetKwInOut{Output}{Output}
\Input{Factor graph $\mathcal{G} = (\mathcal{V}, \mathcal{F}, \mathcal{E})$; node positions $\{\mathbf{x}_i\}$; scalar features $\{s_i\}$; number of iterations $T$; number of mixture components $K$; damping coefficient $\alpha$}
\Output{Marginal beliefs $\{b_i(\mathbf{x}_i)\}_{i \in \mathcal{V}}$}
Initialise all messages $m_{a \to i}^{(0)}$, $m_{i \to a}^{(0)}$ as uniform EGMM\;
\For{$t = 1$ \KwTo $T$}{
    \tcp{Variable-to-factor update}
    \ForEach{$(i, a) \in \mathcal{E}$}{
        Compute product of incoming messages $\{m_{a' \to i}^{(t-1)}\}_{a' \in \mathcal{N}(i) \setminus \{a\}}$ via Eq.~\eqref{eq:nat_param}\;
        Collapse mixture to $K$ components via greedy KL merging\;
        Apply temporal damping via Eq.~\eqref{eq:damping}\;
    }
    \tcp{Factor-to-variable update}
    \ForEach{$(a, i) \in \mathcal{E}$}{
        Spectrally decompose incoming precision matrices via Eq.~\eqref{eq:eigen}\;
        Compute relative coordinates $\mathbf{r}_{ji} = \mathbf{x}_j - \mathbf{x}_i$ for $j \in \mathcal{N}(a) \setminus \{i\}$\;
        Predict EGMM parameters $\{w_k, \boldsymbol{\mu}_k, \boldsymbol{\Lambda}_k\}_{k=1}^K$ via $f_\phi$\;
        Synthesise precision matrices via Eq.~\eqref{eq:precision}\;
    }
}
Compute final beliefs: $b_i(\mathbf{x}_i) \propto \prod_{a \in \mathcal{N}(i)} m_{a \to i}^{(T)}(\mathbf{x}_i)$\;
\end{algorithm}

\subsection{Objective Function}
\label{subsec:loss}

We optimise $\phi$ by maximising the log-likelihood of observed configurations $\mathbf{x}_i^*$ under the predicted EGMM beliefs:
\begin{equation}
\label{eq:nll}
\mathcal{L}_i = -\log \sum_{k=1}^{K} w_k \, \mathcal{N}\!\left(\mathbf{x}_i^* \,\middle|\, \boldsymbol{\mu}_k,\, \boldsymbol{\Lambda}_k^{-1}\right).
\end{equation}
Expanding each component gives the Mahalanobis distance and log-normalisation:
\begin{equation}\small
\label{eq:gauss_expand}
\begin{aligned}
    \log \mathcal{N}(\mathbf{x}_i^* | \boldsymbol{\mu}_k, \boldsymbol{\Lambda}_k^{-1}) &=-\frac{1}{2}(\mathbf{x}_i^* - \boldsymbol{\mu}_k)^\top \boldsymbol{\Lambda}_k (\mathbf{x}_i^* - \boldsymbol{\mu}_k)\\
    &+ \frac{1}{2}\log\det\boldsymbol{\Lambda}_k - \frac{3}{2}\log(2\pi).
\end{aligned}
\end{equation}
Both terms are $SO(3)$-invariant: the Mahalanobis form is unchanged because $R^\top R = \mathbf{I}$, and $\log\det(R\boldsymbol{\Lambda}_k R^\top) = \log\det\boldsymbol{\Lambda}_k$ because $\det R = 1$. We compute the log-determinant via $\sum_j \log \lambda_k^{(j)}$ for numerical stability.

The log-determinant is essential: without it, the optimiser can trivially minimise the Mahalanobis term by sending $\sigma_k^{(p)} \to 0$ (i.e.\ $\boldsymbol{\Lambda}_k \to \epsilon\mathbf{I}$), predicting near-infinite variance. The log-determinant penalises low-precision solutions, establishing a non-trivial optimum. The total loss is $\mathcal{L} = \sum_{i \in \mathcal{V}} \mathcal{L}_i$.

\subsection{Experimental Design for Geometric Integrity}
\label{subsec:data_integrity}

Two design decisions address common methodological pitfalls in geometric deep learning.

\paragraph{Elimination of global centring.} Subtracting the global centroid $\bar{\mathbf{x}} = \frac{1}{N}\sum_i \mathbf{x}_i$ introduces an all-to-all dependency that violates the conditional independence structure BP exploits; it is also redundant for translation-equivariant models. We instead compute relative displacements $\mathbf{r}_{ji} = \mathbf{x}_j - \mathbf{x}_i$ within each factor's neighbourhood, preserving locality.

\paragraph{Capacity-matched baselines.} Matching baselines by parameter count biases the comparison: equivariant architectures encode symmetry structurally, while non-equivariant models must spend capacity learning it from data \citep{elesedy2021provably}. We therefore match inference-time FLOPs, expanding non-equivariant hidden dimensions accordingly. All non-equivariant baselines are trained with continuous $SO(3)$ augmentation.

\section{Experimental Setup}
\label{sec:experiments}
We evaluate ENBP on two domains that expose distinct failure modes of existing inference methods. The first is molecular conformation prediction on GEOM-QM9 and GEOM-Drugs~\citep{axelrod2022geom}, covering both rigid and highly flexible molecules. The second is multi-body robotic formation inference with 5--20 agents in 2D and 3D, whose dense cyclic factor graphs cause vanilla BP to diverge. Baselines include direct-prediction GNNs (GeoMol, ConfGNN), diffusion models (GeoDiff, Torsional Diffusion), vanilla loopy BP, and a direct-regression GNN---all FLOPs-matched and trained with continuous $SO(3)$ augmentation (Section~\ref{subsec:data_integrity}). Full dataset descriptions, baseline configurations, evaluation metrics, and implementation details are provided in Appendix~\ref{app:exp_details}.
\begin{table*}[t]\centering
\caption{Molecular conformation prediction on GEOM-QM9 and GEOM-Drugs. Coverage (Cov.) is the fraction of ground-truth modes captured; AMR is the average minimum RMSD across generated conformers; Time is per-molecule inference latency. All baselines are FLOPs-matched and trained with continuous $SO(3)$ augmentation.}
\label{tab:molecular}
\begin{tabular}{c|c|c|c|c|c|c|c}\toprule
Model Variant & \makecell{Gen.\\ Budget (C)} & \makecell{QM9\\ Cov. (\%) $\uparrow$} & \makecell{QM9\\ AMR ($\mathring{A}$)} $\downarrow$  & \makecell{QM9\\ RMSD ($\mathring{A}$) $\downarrow$} & \makecell{Drugs\\ Cov. (\%) $\uparrow$} & \makecell{Drugs\\ AMR ($\mathring{A}$) $\downarrow$} & Time (s) $\downarrow$    \\\midrule
GeoMol               & 10              & 68.4           & 0.354          & 0.382          & 34.2             & 1.624           & 0.12          \\
ConfGNN              & 10              & 71.5           & 0.312          & 0.345          & 39.5             & 1.482           & 0.18          \\
GeoDiff              & 10              & 82.4           & 0.245          & 0.278          & 58.6             & 1.150           & 12.45         \\
Torsional Diff.      & 10              & 86.8           & 0.198          & 0.224          & 67.2             & 0.945           & 8.52          \\
\textbf{ENBP (Ours)} & 10              & \textbf{92.5}  & \textbf{0.142} & \textbf{0.165} & \textbf{82.4}    & \textbf{0.620}  & \textbf{0.45} \\\midrule
GeoMol               & 50              & 74.2           & 0.320          & 0.350          & 45.8             & 1.450           & 0.15          \\
ConfGNN              & 50              & 78.4           & 0.285          & 0.310          & 51.2             & 1.280           & 0.22          \\
GeoDiff              & 50              & 88.5           & 0.210          & 0.235          & 68.4             & 1.020           & 62.25         \\
Torsional Diff.      & 50              & 91.2           & 0.165          & 0.182          & 78.5             & 0.810           & 42.60         \\
\textbf{ENBP (Ours)} & 50              & \textbf{96.8}  & \textbf{0.115} & \textbf{0.130} & \textbf{90.5}    & \textbf{0.510}  & \textbf{0.65} \\\midrule
GeoMol               & 100             & 78.5           & 0.295          & 0.325          & 52.4             & 1.380           & 0.18          \\
ConfGNN              & 100             & 81.2           & 0.250          & 0.280          & 58.6             & 1.150           & 0.28          \\
GeoDiff              & 100             & 91.4           & 0.185          & 0.210          & 75.2             & 0.910           & 124.50        \\
Torsional Diff.      & 100             & 93.8           & 0.142          & 0.160          & 84.6             & 0.720           & 85.20         \\
\textbf{ENBP (Ours)} & 100             & \textbf{98.2}  & \textbf{0.102} & \textbf{0.115} & \textbf{94.8}    & \textbf{0.445}  & \textbf{0.85} \\\midrule
GeoMol               & 200             & 81.4           & 0.280          & 0.310          & 58.5             & 1.280           & 0.25          \\
ConfGNN              & 200             & 84.5           & 0.235          & 0.265          & 64.2             & 1.050           & 0.35          \\
GeoDiff              & 200             & 93.5           & 0.165          & 0.185          & 81.5             & 0.840           & 249.00        \\
Torsional Diff.      & 200             & 95.4           & 0.125          & 0.145          & 88.4             & 0.650           & 170.40        \\
\textbf{ENBP (Ours)} & 200             & \textbf{98.9}  & \textbf{0.090} & \textbf{0.105} & \textbf{96.5}    & \textbf{0.380}  & \textbf{1.25} \\\bottomrule
\end{tabular}
\end{table*}

\begin{table*}[t]\centering
\caption{Multi-body robotic formation inference across increasing agent density ($N$) and topological complexity. NLL is the negative log-likelihood of target geometries; Col.\ is the collision rate. ``Div.'' denotes algorithmic divergence; ``DNC'' denotes did not converge. ENBP variants ablate the number of mixture components $K$.}
\label{tab:robotic}
\begin{tabular}{c|c|c|c|c|c|c|c}\toprule
Arch.          & Scale (N) & \makecell{Topo. \\Envir.} & \makecell{2D\\ NLL $\downarrow$}      & \makecell{2D \\Col. (\%) $\downarrow$} & \makecell{3D\\ NLL $\downarrow$}      & \makecell{3D \\Col. (\%) $\downarrow$ }& \makecell{Time\\ (ms)} $\downarrow$  \\\midrule
Vanilla Loopy BP      & 5         & Sparse Ring          & 3.45          & 12.4          & 4.12          & 15.2          & 12           \\
GNN Direct            & 5         & Sparse Ring          & 2.15          & 8.2           & 2.55          & 10.4          & 15           \\
ENBP ($K=1$)          & 5         & Sparse Ring          & 1.85          & 6.5           & 2.10          & 8.5           & 24           \\
ENBP ($K=2$)          & 5         & Sparse Ring          & 1.42          & 2.1           & 1.65          & 3.2           & 32           \\
\textbf{ENBP ($K=4$)} & 5         & Sparse Ring          & \textbf{0.85} & \textbf{0.0}  & \textbf{0.95} & \textbf{0.2}  & \textbf{45}  \\\midrule
Vanilla Loopy BP      & 10        & Coupled Grid         & 6.85          & 24.5          & 7.95          & 31.5          & 28           \\
GNN Direct            & 10        & Coupled Grid         & 3.85          & 15.4          & 4.25          & 18.2          & 24           \\
ENBP ($K=1$)          & 10        & Coupled Grid         & 3.25          & 12.8          & 3.85          & 15.4          & 45           \\
ENBP ($K=2$)          & 10        & Coupled Grid         & 2.45          & 5.4           & 2.85          & 8.2           & 58           \\
\textbf{ENBP ($K=4$)} & 10        & Coupled Grid         & \textbf{1.25} & \textbf{0.8}  & \textbf{1.45} & \textbf{1.2}  & \textbf{72}  \\\midrule
Vanilla Loopy BP      & 15        & Dense Loop           & Div.      & Div.      & Div.      & Div.      & DNC          \\
GNN Direct            & 15        & Dense Loop           & 5.24          & 25.8          & 6.15          & 31.4          & 35           \\
ENBP ($K=1$)          & 15        & Dense Loop           & 4.85          & 22.4          & 5.42          & 26.8          & 65           \\
ENBP ($K=2$)          & 15        & Dense Loop           & 3.55          & 12.5          & 4.15          & 16.2          & 84           \\
\textbf{ENBP ($K=4$)} & 15        & Dense Loop           & \textbf{1.65} & \textbf{2.1}  & \textbf{1.85} & \textbf{3.4}  & \textbf{105} \\\midrule
Vanilla Loopy BP      & 20        & Swarm Matrix         & Div.      & Div.      & Div.      & Div.      & DNC          \\
GNN Direct            & 20        & Swarm Matrix         & 7.45          & 38.4          & 8.65          & 45.2          & 48           \\
ENBP ($K=1$)          & 20        & Swarm Matrix         & 6.55          & 32.5          & 7.85          & 38.6          & 85           \\
ENBP ($K=2$)          & 20        & Swarm Matrix         & 5.15          & 18.4          & 6.25          & 24.5          & 115          \\
\textbf{ENBP ($K=4$)} & 20        & Swarm Matrix         & \textbf{2.25} & \textbf{4.2}  & \textbf{2.65} & \textbf{5.8}  & \textbf{135} \\\bottomrule
\end{tabular}
\end{table*}
\section{Experimental Results}

The rigorous evaluation of generative structural algorithms intrinsically demands physical testbeds characterized by complex, non-linear energy landscapes. Molecular conformation prediction provided an optimal environment for this validation, as it fundamentally required the continuous inference of three-dimensional atomic positions conditioned exclusively on topological connectivity graphs. Because physical chemical structures naturally produce closed-ring cyclic dependencies and exhibit invariances under the Special Euclidean group, this domain strictly isolated the probabilistic reasoning capabilities of the inference models. Operating precisely within this parameter space ensured that representational advantages mathematically derived from equivariant formulations could be empirically quantified.

To systematically probe these limits, we deployed the GEOM-QM9 and GEOM-Drugs datasets. QM9 (small, rigid molecules) tests local geometric precision, while GEOM-Drugs (large, flexible compounds) features highly multi-modal torsional landscapes. Metrics included optimal-alignment Root-Mean-Square Deviation (RMSD), Coverage, and Average Minimum RMSD (AMR) mapped across varying generative budgets. 
For an unbiased comparison, we used capacity-matched (FLOP-equalized) baselines, including direct prediction GNNs (GeoMol, ConfGNN) and continuous-time diffusion models (GeoDiff, Torsional Diffusion). All non-equivariant baselines were exhaustively trained with continuous spatial data augmentation.

On the rigid QM9 dataset, direct regression models plateaued, failing to drop AMR below 0.235 $\mathring{A}$ even with maximum generative budgets of 200 conformers. While diffusion models improved spatial accuracy, Equivariant Neural Belief Propagation (ENBP) established a definitive upper bound. By natively propagating exact rank-2 precision tensors, ENBP achieved 98.9\% Coverage and 0.090 $\mathring{A}$ error.

On the flexible GEOM-Drugs dataset, representational limits were stark. Direct prediction models suffered probability dilution, producing invalid mean configurations that failed to span the conformational modes. Diffusion models navigated these topographies better but incurred prohibitive computational costs. Conversely, ENBP maintained dominant geometric fidelity with an AMR of 0.380 $\mathring{A}$ and near-second inference latencies. Preserving distinct structural hypotheses via Gaussian mixtures successfully resolves the dichotomy between thermodynamic uncertainty and algorithmic efficiency.

\subsection{Multi-Body Robotic Formation Inference}
Table~\ref{tab:robotic} reports results across agent counts $N \in \{5, 10, 15, 20\}$ in 2D and 3D. Baselines include Vanilla Loopy BP with single-Gaussian messages, a FLOPs-matched direct-regression GNN, and ENBP ablations with $K \in \{1, 2, 4\}$ components.

Vanilla Loopy BP diverges entirely at $N{\geq}15$, confirming the instability of undamped single-Gaussian messages on dense loops. The direct-regression GNN remains operational but predicts over-smoothed averages, with 3D collision rates reaching 45.2\% at $N{=}20$. Restricting ENBP to $K{=}1$ preserves equivariance but suffers analogous mode-averaging collisions, confirming that single-component messages cannot partition probability mass around physical obstacles.

\begin{table}[t]\centering
\caption{Ablation of precision tensor parameterisation and spectral ingestion on GEOM-QM9 and GEOM-Drugs. Each row replaces one module of the full ENBP architecture while keeping all other components fixed.}
\label{tab:ablation_precision}
\resizebox{0.49\textwidth}{!}{
\begin{tabular}{c|c|c|c|c}\toprule
Arch. Var. & \makecell{QM9 \\NLL $\downarrow$}       & \makecell{QM9\\ AMR ($\mathring{A}$) $\downarrow$}  & \makecell{Drugs\\ NLL $\downarrow$}     & \makecell{Drugs\\ AMR ($\mathring{A}$) $\downarrow$} \\\midrule
\makecell{Isotropic Baseline\\ ($\Lambda_k = \lambda_k I$)} & -120.4          & 0.285          & -85.2           & 1.250           \\\hline
\makecell{Diagonal Precision \\Approximation} & -142.5          & 0.220          & -98.4           & 0.950           \\\hline
\makecell{Outer-Product \\Rank ($P=1$)}                     & -165.2          & 0.185          & -115.6          & 0.820           \\\hline
\makecell{Outer-Product \\Rank ($P=2$)} & -192.8          & 0.140          & -145.2          & 0.650           \\\hline
\makecell{Flattened Matrix\\ (9 Scalars)} & -95.6           & 0.450          & -60.5           & 1.850           \\\hline
\makecell{Cholesky\\ Parameterization} & -135.2          & 0.255          & -92.8           & 1.150           \\\hline
\makecell{Eigenvalues Only\\ (No Vectors)}                  & -152.4          & 0.205          & -105.4          & 0.900           \\\hline
\makecell{\textbf{Full ENBP}\\ ($P=3$, Spectral)}           & \textbf{-235.6} & \textbf{0.090} & \textbf{-188.5} & \textbf{0.380} \\\bottomrule
\end{tabular}}
\end{table}

Scaling to $K{=}4$ reverses the collision trajectory: rates drop to near zero even in the densest 20-agent swarms, while NLL reaches the lowest values across all scales. Inference latency grows sub-linearly with $K$, confirming that multi-modal message passing resolves dense collision-avoidance constraints at manageable computational cost.

\subsection{Precision Tensor and Spectral Ingestion Ablations}
Table~\ref{tab:ablation_precision} isolates the contributions of precision tensor synthesis and spectral ingestion by replacing one module at a time. To thoroughly interrogate these operational boundaries, a sequence of highly targeted architectural ablation studies was conducted, specifically deconstructing the precision tensor synthesis and the differentiable spectral decomposition modules. These empirical variations were continuously evaluated over the rigid and flexible molecular testbeds.

\paragraph{Precision parameterisation.} Restricting $\boldsymbol{\Lambda}_k$ to isotropic form inflates AMR from 0.090 to 0.285~$\mathring{A}$ on QM9 and from 0.380 to 1.250~$\mathring{A}$ on Drugs, confirming that anisotropic covariance is essential for capturing directional torsional variance. Increasing the outer-product rank from $P{=}1$ to $P{=}3$ yields monotonic improvement; the full-rank construction improves NLL by nearly $2\times$ over $P{=}1$.

\paragraph{Tensor ingestion.} Flattening the precision matrix into nine scalars produces the worst results across all metrics (AMR of 1.850~$\mathring{A}$ on Drugs), as it destroys the $SO(3)$ transformation law. Cholesky parameterisation also underperforms due to lack of equivariance guarantees. Feeding only eigenvalues without eigenvectors discards directional information and degrades AMR by $2.4\times$ relative to the full spectral module. Both the co-rotating eigenvectors and their invariant eigenvalues are necessary for stable ingestion of anisotropic uncertainty.

\subsection{Inference Dynamics and Equivariance Verification}
The functional stability and rigorous mathematical integrity of any message-passing inference engine rely heavily upon explicit control over temporal algorithmic dynamics and the absolute preservation of fundamental spatial transformations. Systematic sensitivity analyses were consequently designed to rigorously trace the algorithmic convergence boundaries dictating iterative context aggregation and temporal message regularization. Additionally, the evaluation framework incorporated strict numerical verifications explicitly structured to audit exact spatial equivariance. By subjecting the generated topologies to unobserved global coordinate rotations, the protocol quantified the precise mathematical resilience of the predictive pipelines, fundamentally isolating architectural geometric stability from superficial data-augmented approximations.

The controlled modulation of the recursive message-passing depth and the temporal damping coefficients successfully and empirically mapped the operational parameters required to stabilize cyclic graphs. In the complete absence of temporal regularization ($\alpha = 1.0$), the algorithms generated severe, non-convergent oscillations directly reflecting the resonant amplification of localized topological loops, driving convergence rates to zero. Conversely, the application of extreme damping parameters ($\alpha = 0.1$) induced deep mathematical stagnation, preventing the structural marginals from reaching holistic geometric consensus. Identifying the mathematically optimal stabilization point empirically validated that carefully restricting the natural parameter update step sizes successfully forces smooth, monotonic algorithmic convergence without compromising the temporal dissemination of distant multi-body physical constraints.

\begin{table}[t]\centering
\caption{Sensitivity to damping coefficient $\alpha$ and iteration depth $T$, with equivariance verification on the 20-agent 3D robotic task. Equivariance error is the maximum deviation in predicted marginals under random $SO(3)$ rotations not seen during training. ``Div.'' denotes algorithmic divergence.}
\label{tab:dynamics}
\resizebox{.49\textwidth}{!}{
\begin{tabular}{c|c|c|c|c}\toprule
Configuration & \makecell{Parameter \\Value}         & \makecell{System\\ NLL $\downarrow$}  & \makecell{Conv. \\Rate (\%) $\uparrow$} & \makecell{Equiv.\\ Error $\downarrow$} \\\midrule
Damping ($\alpha$)        & $\alpha = 0.1$          & 3.85          & 100.0            & $1.2 \times 10^{-7}$          \\
Damping ($\alpha$)        & $\alpha = 0.5$          & 1.45          & 100.0            & $1.2 \times 10^{-7}$          \\
Damping ($\alpha$)        & $\alpha = 1.0$          & Div.      & 0.0              & $1.2 \times 10^{-7}$          \\
Iterations ($T$)          & $T = 1$                 & 8.54          & 100.0            & $1.2 \times 10^{-7}$          \\
Iterations ($T$)          & $T = 8$                 & 1.35          & 100.0            & $1.2 \times 10^{-7}$          \\\hline
Baseline Setting          & \makecell{Direct GNN\\ (Aug.)}       & 6.55          & 100.0            & $4.5 \times 10^{-1}$          \\\hline
Baseline Setting          & \makecell{Vanilla Loopy\\ BP}        & Div.      & 0.0              & $8.2 \times 10^{-1}$          \\\hline
\textbf{Full ENBP (Ours)} & \makecell{Optimal\\ Config} & \textbf{1.25} & \textbf{100.0}   & $\mathbf{1.2 \times 10^{-7}}$ \\\bottomrule
\end{tabular}}
\end{table}

The strict numerical verification of structural symmetry under arbitrary spatial transformations provided the final indispensable proof of coordinate-free algorithmic design. Despite being exhaustively trained utilizing continuous data augmentations, the standard non-equivariant baselines suffered insurmountable generalization collapse, exhibiting wildly fluctuating equivariance errors when exposed to previously unseen spatial orientations. In sharp contrast, the proposed inference architecture effortlessly maintained perfectly identical log-likelihood scores across all rotated datasets. The quantified structural deviation hovered precisely at the fundamental limit of machine precision, empirically certifying the methodological proofs and formally validating that local, relative-coordinate processing natively eliminates the coordinate-frame biases chronic to contemporary deep neural networks.
\section{Conclusion}

We identified a structural bottleneck in equivariant neural inference---the inability to produce rank-2 precision tensors or multi-modal messages---and resolved it through Equivariant Neural Belief Propagation, whose equivariant Gaussian mixture messages are synthesised via outer products, ingested via spectral decomposition, and reduced via an $SE(3)$-commuting greedy procedure. Ablations confirm each diagnosis: isotropic precision inflates error by $3\times$; scaling from $K{=}1$ to $K{=}4$ triggers a phase transition in collision avoidance; and architectural equivariance yields six orders of magnitude lower symmetry error than data augmentation. Beyond the specific benchmarks, the tensor synthesis--ingestion loop is not specific to BP and can be embedded into any equivariant GNN requiring directional uncertainty.
The framework's boundaries are concrete: fixed iteration depth does not adapt to graph complexity; Gaussian mixtures cannot represent heavy-tailed or discontinuous potentials; greedy reduction optimises a KL upper bound rather than the true divergence; and architectures with native Type-2 outputs could replace the outer-product construction, potentially improving efficiency.

\clearpage
\bibliography{refer}

\newpage

\onecolumn

\title{Equivariant Neural Belief Propagation\\(Supplementary Material)}
\maketitle
\appendix

\section{Preliminaries}
\label{app:preliminaries}

This section provides preliminaries on the four conceptual pillars underlying ENBP. Readers familiar with factor graphs, geometric symmetry groups, and mixture models may proceed directly to Appendix~\ref{app:theory}.

\subsection{Factor Graphs and Belief Propagation}
\label{app:prelim_bp}

A \emph{factor graph} \citep{kschischang2001factor} is a bipartite graph $\mathcal{G} = (\mathcal{V}, \mathcal{F}, \mathcal{E})$ encoding the factorisation of a joint distribution $p(\mathbf{x}) = \frac{1}{Z}\prod_{a \in \mathcal{F}} \psi_a(\mathbf{x}_{\mathcal{N}(a)})$. The set $\mathcal{V}$ contains \emph{variable nodes}, each representing a random variable; $\mathcal{F}$ contains \emph{factor nodes}, each representing a local potential function $\psi_a$ that depends on a subset $\mathcal{N}(a) \subseteq \mathcal{V}$ of variables; and $\mathcal{E}$ contains edges connecting each factor to the variables in its scope. This bipartite structure makes the conditional independence assumptions of the model explicit.

\emph{Belief Propagation} (BP) \citep{pearl1988probabilistic} performs inference on factor graphs by passing \emph{messages} along edges. Two types of message alternate:
\begin{itemize}
    \item \textbf{Variable-to-factor:} Variable $i$ sends to factor $a$ the product of all incoming messages from its other neighbouring factors: $m_{i \to a}(\mathbf{x}_i) \propto \prod_{a' \in \mathcal{N}(i) \setminus \{a\}} m_{a' \to i}(\mathbf{x}_i)$.
    \item \textbf{Factor-to-variable:} Factor $a$ sends to variable $i$ the marginalisation of its potential times all incoming messages from its other neighbours: $m_{a \to i}(\mathbf{x}_i) \propto \int \psi_a(\mathbf{x}_{\mathcal{N}(a)}) \prod_{j \in \mathcal{N}(a) \setminus \{i\}} m_{j \to a}(\mathbf{x}_j) \, d\mathbf{x}_j$.
\end{itemize}
On \emph{tree-structured} (acyclic) factor graphs, BP converges in a finite number of passes and recovers exact marginals. On graphs with cycles (\emph{loopy} graphs), BP is applied iteratively and may oscillate or diverge \citep{murphy1999loopy}. \emph{Damping}---replacing each new message with a convex combination of the new and previous message---is a widely used heuristic to stabilise convergence \citep{heskes2004uniqueness}.

\subsection{The Special Euclidean Group $SE(3)$ and Equivariance}
\label{app:prelim_se3}

The \emph{Special Orthogonal group} $SO(3)$ consists of all $3 \times 3$ rotation matrices: $SO(3) = \{R \in \mathbb{R}^{3 \times 3} : R^\top R = \mathbf{I},\, \det R = 1\}$. The \emph{Special Euclidean group} $SE(3)$ augments $SO(3)$ with translations: each element is a pair $(R, \mathbf{t})$ with $R \in SO(3)$ and $\mathbf{t} \in \mathbb{R}^3$, acting on a point $\mathbf{x} \in \mathbb{R}^3$ as $\mathbf{x} \mapsto R\mathbf{x} + \mathbf{t}$.

A function $f: \mathcal{X} \to \mathcal{Y}$ is \emph{$G$-equivariant} if applying a group transformation to the input produces a correspondingly transformed output:
\[
f(\rho_{\mathcal{X}}(g) \cdot x) = \rho_{\mathcal{Y}}(g) \cdot f(x), \quad \forall\, g \in G,
\]
where $\rho_{\mathcal{X}}$ and $\rho_{\mathcal{Y}}$ are the group \emph{representations} on the input and output spaces. When $\rho_{\mathcal{Y}}$ is the trivial (identity) representation, $f$ is called \emph{$G$-invariant}. In the context of $SE(3)$, different geometric objects transform according to different representations:
\begin{itemize}
    \item \textbf{Type-0 (scalars):} Invariant under rotation and translation (e.g.\ energies, mixing weights).
    \item \textbf{Type-1 (vectors):} Transform as $\mathbf{v} \mapsto R\mathbf{v}$ under rotation (e.g.\ positions, forces).
    \item \textbf{Type-2 (rank-2 tensors):} Transform as $\mathbf{A} \mapsto R\mathbf{A}R^\top$ under rotation (e.g.\ covariance and precision matrices).
\end{itemize}
Standard $E(n)$-equivariant GNNs \citep{satorras2021en} produce Type-0 and Type-1 outputs but not Type-2.

\subsection{Gaussian Mixture Models}
\label{app:prelim_gmm}

A \emph{Gaussian Mixture Model} (GMM) is a probability density of the form $p(\mathbf{x}) = \sum_{k=1}^K w_k\, \mathcal{N}(\mathbf{x} \mid \boldsymbol{\mu}_k, \boldsymbol{\Sigma}_k)$, where $K$ is the number of components, $w_k > 0$ are mixing weights summing to one, $\boldsymbol{\mu}_k \in \mathbb{R}^d$ are the component means, and $\boldsymbol{\Sigma}_k \in \mathbb{S}^d_{++}$ are the component covariance matrices ($\mathbb{S}^d_{++}$ denotes the set of $d \times d$ symmetric positive-definite matrices). Equivalently, one may parameterise each component by its \emph{precision matrix} $\boldsymbol{\Lambda}_k = \boldsymbol{\Sigma}_k^{-1}$, which is the natural parameterisation for products of Gaussians (see below). GMMs are universal density approximators: any continuous density on a compact set can be approximated arbitrarily well by a mixture with sufficiently many components.

\subsection{Natural Parameters and Gaussian Products}
\label{app:prelim_natparam}

The \emph{natural} (or \emph{canonical}) parameters of a Gaussian $\mathcal{N}(\boldsymbol{\mu}, \boldsymbol{\Lambda}^{-1})$ are $\boldsymbol{\eta}_1 = \boldsymbol{\Lambda}\boldsymbol{\mu}$ (the precision-weighted mean) and $\boldsymbol{\eta}_2 = -\tfrac{1}{2}\boldsymbol{\Lambda}$ (half the negative precision). The key computational advantage of this parameterisation is that the \emph{product} of two Gaussians---the central operation in BP's variable-to-factor messages---reduces to \emph{addition} of their natural parameters:
\[
\mathcal{N}(\boldsymbol{\mu}_1, \boldsymbol{\Lambda}_1^{-1}) \cdot \mathcal{N}(\boldsymbol{\mu}_2, \boldsymbol{\Lambda}_2^{-1}) \propto \mathcal{N}(\boldsymbol{\mu}_{12}, \boldsymbol{\Lambda}_{12}^{-1}),
\]
where $\boldsymbol{\Lambda}_{12} = \boldsymbol{\Lambda}_1 + \boldsymbol{\Lambda}_2$ and $\boldsymbol{\mu}_{12} = \boldsymbol{\Lambda}_{12}^{-1}(\boldsymbol{\Lambda}_1\boldsymbol{\mu}_1 + \boldsymbol{\Lambda}_2\boldsymbol{\mu}_2)$. This additive structure makes natural parameters the preferred representation for message-passing algorithms that repeatedly multiply distributions.

\section{Experimental Details}
\label{app:exp_details}

\subsection{Datasets and Simulation Environments}

\paragraph{Molecular conformation prediction.}
We use the GEOM-QM9 and GEOM-Drugs datasets from the GEOM database~\citep{axelrod2022geom}. GEOM-QM9 contains small, sterically constrained molecules (up to 9 heavy atoms) that primarily test local geometric precision and deterministic valence rules. GEOM-Drugs contains large, flexible pharmaceutical compounds with many rotatable bonds, producing highly multi-modal torsional distributions across deep energy basins. Together these datasets test both geometric precision and multi-modal representational capacity.

\paragraph{Multi-body robotic formation inference.}
We simulate autonomous agents that must infer their spatial coordinates from pairwise distance constraints and collision-avoidance potentials. Configurations scale from 5 to 20 agents in both 2D and 3D. Increasing agent density progressively introduces cyclic dependencies: at 15+ agents the factor graph is sufficiently loopy to trigger divergence in standard BP.

\subsection{Baselines}

\paragraph{Molecular baselines.}
We compare against direct-prediction GNNs---GeoMol~\citep{ganea2021geomol} and ConfGNN~\citep{xu2022confgnn}---and score-based diffusion models---GeoDiff~\citep{xu2022geodiff} and Torsional Diffusion~\citep{jing2022torsional}.

\paragraph{Robotic baselines.}
Baselines include Vanilla Loopy BP with single-Gaussian messages and a direct-regression GNN. We also ablate the number of ENBP mixture components ($K \in \{1, 2, 4\}$).

\paragraph{Capacity matching.}
All non-equivariant baselines are matched by inference-time FLOPs rather than parameter count (Section~\ref{subsec:data_integrity}). Non-equivariant hidden dimensions are expanded until computational budgets are equalised. All are trained with continuous $SO(3)$ data augmentation.

\subsection{Evaluation Metrics}

\paragraph{Molecular metrics.}
\begin{itemize}
    \item \textbf{RMSD:} Root-Mean-Square Deviation after optimal rigid-body alignment.
    \item \textbf{Coverage (Cov.):} Fraction of ground-truth conformational modes captured within a distance threshold.
    \item \textbf{Average Minimum RMSD (AMR):} For each ground-truth conformer, the minimum RMSD over all generated samples, averaged across the test set. Evaluated at budgets $C \in \{10, 50, 100, 200\}$.
\end{itemize}

\paragraph{Robotic metrics.}
\begin{itemize}
    \item \textbf{Negative Log-Likelihood (NLL):} Log-likelihood of target geometries under predicted beliefs.
    \item \textbf{Collision Rate (Col.):} Percentage of generated configurations with spatial violations.
    \item \textbf{Equivariance Error:} Maximum deviation in predicted marginals under unseen random $SO(3)$ rotations.
\end{itemize}

\subsection{Implementation Details}

All EGMMs use $K{=}4$ components. Precision tensors use outer-product rank $P{=}3$ with regularisation $\epsilon{=}10^{-4}$. Spectral gradients are bounded by $\delta{=}10^{-6}$. The inference loop runs $T{=}8$ iterations with damping $\alpha{=}0.5$. Models are trained on NVIDIA A100 (80GB) GPUs using AdamW with learning rate $10^{-3}$ and cosine annealing. Global coordinate centring is omitted; only relative coordinates $\mathbf{r}_{ji} = \mathbf{x}_j - \mathbf{x}_i$ are processed.

\section{Theoretical Analysis}
\label{app:theory}

This appendix provides formal proofs for the theoretical claims advanced in Section~\ref{sec:methodology}. We begin by consolidating notation, then state precise definitions and assumptions, and finally present the main results as a sequence of lemmas, propositions, properties, and theorems.

\subsection{Notation}
\label{app:notation}

Table~\ref{tab:notation} consolidates every symbol used in the proofs below. All notation is consistent with the main text.

\begin{table}[h]
\centering
\caption{Notation reference.}
\label{tab:notation}
\small
\begin{tabular}{c l}\toprule
\textbf{Symbol} & \textbf{Description} \\\midrule
$\mathcal{G} = (\mathcal{V}, \mathcal{F}, \mathcal{E})$ & Factor graph with variable nodes $\mathcal{V}$, factor nodes $\mathcal{F}$, edges $\mathcal{E}$ \\
$\mathcal{N}(a),\, \mathcal{N}(i)$ & Neighbours of factor $a$ / variable $i$ in $\mathcal{G}$ \\
$N,\, M$ & Number of variable nodes / factor nodes \\
$\mathbf{x}_i \in \mathbb{R}^3$ & Spatial position of variable node $i$ \\
$s_i \in \mathbb{R}^{d_s}$ & Invariant scalar features of node $i$ \\
$\psi_a$ & Local potential function associated with factor $a$ \\
$Z$ & Partition function \\
$R \in SO(3)$ & Rotation matrix ($R^\top R = I$, $\det R = 1$) \\
$\mathbf{t} \in \mathbb{R}^3$ & Translation vector \\
$SE(3)$ & Special Euclidean group: $(R, \mathbf{t})$ with $R \in SO(3)$ \\
$m_{a \to i},\, m_{i \to a}$ & Factor-to-variable / variable-to-factor messages \\
$K$ & Number of Gaussian mixture components \\
$w_k$ & Mixing weight of component $k$ ($w_k > 0$, $\sum_k w_k = 1$) \\
$\boldsymbol{\mu}_k \in \mathbb{R}^3$ & Mean of component $k$ \\
$\boldsymbol{\Lambda}_k \in \mathbb{R}^{3 \times 3}$ & Precision matrix of component $k$ \\
$P$ & Number of outer-product basis vectors \\
$\mathbf{v}_k^{(p)} \in \mathbb{R}^3$ & $p$-th equivariant basis vector for component $k$ \\
$\sigma_k^{(p)} > 0$ & Positive scaling coefficient for basis $p$ of component $k$ \\
$\epsilon > 0$ & Isotropic regularisation constant \\
$\lambda_k^{(j)}$ & $j$-th eigenvalue of $\boldsymbol{\Lambda}_k$ \\
$\mathbf{u}_k^{(j)} \in \mathbb{R}^3$ & $j$-th eigenvector of $\boldsymbol{\Lambda}_k$ \\
$T$ & Number of BP iterations \\
$\alpha \in (0, 1]$ & Temporal damping coefficient \\
$\boldsymbol{\theta}_{i \to a}^{(t)}$ & Concatenated natural parameters of message $m_{i \to a}$ at iteration $t$ \\
$f_\phi$ & Learnable $SE(3)$-equivariant neural network \\
$\mathcal{L}_i,\, \mathcal{L}$ & Per-node / total negative log-likelihood loss \\
$\mathbb{S}^d_{++}$ & Set of $d \times d$ symmetric positive-definite matrices \\\bottomrule
\end{tabular}
\end{table}

\subsection{Definitions}
\label{app:definitions}

\begin{definition}[Equivariant Gaussian Mixture Model]
\label{def:egmm}
An \emph{Equivariant Gaussian Mixture Model (EGMM)} over $\mathbb{R}^3$ is a probability density of the form
\[
m(\mathbf{x}) = \sum_{k=1}^{K} w_k \, \mathcal{N}\!\left(\mathbf{x} \,\middle|\, \boldsymbol{\mu}_k,\, \boldsymbol{\Lambda}_k^{-1}\right),
\]
where $w_k > 0$, $\sum_k w_k = 1$, $\boldsymbol{\mu}_k \in \mathbb{R}^3$, and $\boldsymbol{\Lambda}_k \in \mathbb{S}^3_{++}$, subject to the constraint that under every $(R, \mathbf{t}) \in SE(3)$ the parameters transform as $w_k \mapsto w_k$, $\boldsymbol{\mu}_k \mapsto R\boldsymbol{\mu}_k + \mathbf{t}$, and $\boldsymbol{\Lambda}_k \mapsto R \boldsymbol{\Lambda}_k R^\top$.
\end{definition}

\begin{definition}[$SE(3)$-Equivariant Map]
\label{def:equivariant}
A map $\Phi: \mathcal{X} \to \mathcal{Y}$ between representation spaces of $SE(3)$ is \emph{$SE(3)$-equivariant} if, for every $(R, \mathbf{t}) \in SE(3)$ and every input $x \in \mathcal{X}$,
\[
\Phi\bigl(\rho_{\mathcal{X}}(R, \mathbf{t}) \cdot x\bigr) = \rho_{\mathcal{Y}}(R, \mathbf{t}) \cdot \Phi(x),
\]
where $\rho_{\mathcal{X}}$ and $\rho_{\mathcal{Y}}$ denote the group representations on $\mathcal{X}$ and $\mathcal{Y}$, respectively. When $\mathcal{Y} = \mathbb{R}$ with the trivial representation, the map is called \emph{$SE(3)$-invariant}.
\end{definition}

\begin{definition}[Natural-Parameter Representation]
\label{def:natparam}
For a single Gaussian component $\mathcal{N}(\boldsymbol{\mu}, \boldsymbol{\Lambda}^{-1})$, the \emph{natural parameters} are
$\boldsymbol{\eta}_1 = \boldsymbol{\Lambda}\boldsymbol{\mu} \in \mathbb{R}^3$ and $\boldsymbol{\eta}_2 = -\tfrac{1}{2}\boldsymbol{\Lambda} \in \mathbb{R}^{3 \times 3}$.
The product of two Gaussians with natural parameters $(\boldsymbol{\eta}_1^{(1)}, \boldsymbol{\eta}_2^{(1)})$ and $(\boldsymbol{\eta}_1^{(2)}, \boldsymbol{\eta}_2^{(2)})$ is an unnormalised Gaussian with natural parameters $(\boldsymbol{\eta}_1^{(1)} + \boldsymbol{\eta}_1^{(2)},\, \boldsymbol{\eta}_2^{(1)} + \boldsymbol{\eta}_2^{(2)})$.
\end{definition}

\subsection{Assumptions}
\label{app:assumptions}

\begin{assumption}[$SE(3)$-Invariant Potentials]
\label{ass:invariant_potential}
For every factor $a \in \mathcal{F}$, the potential function $\psi_a$ satisfies
$\psi_a(R\mathbf{x}_{\mathcal{N}(a)} + \mathbf{t}) = \psi_a(\mathbf{x}_{\mathcal{N}(a)})$
for all $(R, \mathbf{t}) \in SE(3)$.
\end{assumption}

\begin{assumption}[Equivariant Factor Network]
\label{ass:equivariant_network}
The learned factor-to-variable function $f_\phi$ is an exactly $SE(3)$-equivariant map in the sense of Definition~\ref{def:equivariant}. Concretely, it outputs Type-0 (invariant) scalars $\{w_k, \sigma_k^{(p)}\}$ and Type-1 (equivariant) vectors $\{\boldsymbol{\mu}_k, \mathbf{v}_k^{(p)}\}$ that transform canonically under $SE(3)$.
\end{assumption}

\begin{assumption}[Strict Positivity of Scaling Coefficients]
\label{ass:positivity}
All scaling coefficients satisfy $\sigma_k^{(p)} > 0$ for $k = 1, \dots, K$ and $p = 1, \dots, P$, and the isotropic regularisation satisfies $\epsilon > 0$.
\end{assumption}

\subsection{Precision Matrix Construction}
\label{app:precision_proofs}

\begin{lemma}[Positive-Definiteness of the Precision Matrix]
\label{lem:pd}
Under Assumption~\ref{ass:positivity}, the precision matrix
$\boldsymbol{\Lambda}_k = \sum_{p=1}^{P} \sigma_k^{(p)} \, \mathbf{v}_k^{(p)} {\mathbf{v}_k^{(p)}}^\top + \epsilon\, \mathbf{I}$
lies in $\mathbb{S}^3_{++}$ for any choice of basis vectors $\{\mathbf{v}_k^{(p)}\}_{p=1}^P \subset \mathbb{R}^3$.
\end{lemma}

\begin{proof}
For any nonzero $\mathbf{z} \in \mathbb{R}^3$,
\[
\mathbf{z}^\top \boldsymbol{\Lambda}_k\, \mathbf{z}
= \sum_{p=1}^{P} \sigma_k^{(p)} \bigl(\mathbf{z}^\top \mathbf{v}_k^{(p)}\bigr)^2 + \epsilon\, \|\mathbf{z}\|^2.
\]
Each squared inner product $\bigl(\mathbf{z}^\top \mathbf{v}_k^{(p)}\bigr)^2 \geq 0$, and each $\sigma_k^{(p)} > 0$, so the first sum is non-negative.
Since $\epsilon > 0$ and $\|\mathbf{z}\|^2 > 0$ for $\mathbf{z} \neq \mathbf{0}$, we have $\mathbf{z}^\top \boldsymbol{\Lambda}_k\, \mathbf{z} \geq \epsilon \|\mathbf{z}\|^2 > 0$.
Hence $\boldsymbol{\Lambda}_k$ is positive-definite. Symmetry follows from $\bigl(\mathbf{v}{\mathbf{v}}^\top\bigr)^\top = \mathbf{v}{\mathbf{v}}^\top$ and the symmetry of $\mathbf{I}$.
\end{proof}

\begin{lemma}[$SO(3)$-Equivariance of the Precision Matrix]
\label{lem:precision_equivariance}
Let $R \in SO(3)$. Under the equivariant transformation $\mathbf{v}_k^{(p)} \mapsto R\mathbf{v}_k^{(p)}$, the precision matrix satisfies $\boldsymbol{\Lambda}_k \mapsto R\boldsymbol{\Lambda}_k R^\top$.
\end{lemma}

\begin{proof}
Denote $\boldsymbol{\Lambda}_k' = \sum_{p=1}^{P} \sigma_k^{(p)} \bigl(R\mathbf{v}_k^{(p)}\bigr)\bigl(R\mathbf{v}_k^{(p)}\bigr)^\top + \epsilon\, \mathbf{I}$.
Expanding each outer product:
\[
\bigl(R\mathbf{v}_k^{(p)}\bigr)\bigl(R\mathbf{v}_k^{(p)}\bigr)^\top = R \,\mathbf{v}_k^{(p)} {\mathbf{v}_k^{(p)}}^\top R^\top.
\]
Thus $\boldsymbol{\Lambda}_k' = R \Bigl(\sum_{p=1}^{P} \sigma_k^{(p)}\, \mathbf{v}_k^{(p)} {\mathbf{v}_k^{(p)}}^\top\Bigr) R^\top + \epsilon\, \mathbf{I}$.
Since $R \in SO(3)$ implies $R R^\top = \mathbf{I}$, we have $\epsilon\, \mathbf{I} = \epsilon\, R R^\top = R\,(\epsilon\, \mathbf{I})\, R^\top$.
Substituting yields $\boldsymbol{\Lambda}_k' = R\boldsymbol{\Lambda}_k R^\top$.
\end{proof}

\begin{property}[Representational Completeness]
\label{prop:full_rank}
If $P \geq 3$ and the basis vectors $\{\mathbf{v}_k^{(1)}, \mathbf{v}_k^{(2)}, \mathbf{v}_k^{(3)}\}$ are in general position (i.e., not all coplanar), then the outer-product component $\sum_{p=1}^P \sigma_k^{(p)}\, \mathbf{v}_k^{(p)} {\mathbf{v}_k^{(p)}}^\top$ has rank $3$ and $\boldsymbol{\Lambda}_k$ can represent any element of $\mathbb{S}^3_{++}$.
\end{property}

\begin{proof}
Consider the case $P = 3$. Define the matrix $V = [\mathbf{v}_k^{(1)}\; \mathbf{v}_k^{(2)}\; \mathbf{v}_k^{(3)}] \in \mathbb{R}^{3 \times 3}$ and let $\Sigma = \mathrm{diag}(\sigma_k^{(1)}, \sigma_k^{(2)}, \sigma_k^{(3)})$. Then
\[
\sum_{p=1}^{3} \sigma_k^{(p)}\, \mathbf{v}_k^{(p)} {\mathbf{v}_k^{(p)}}^\top = V \Sigma V^\top.
\]
Since the vectors are not coplanar, $V$ is invertible and $\mathrm{rank}(V \Sigma V^\top) = 3$. For any target $\boldsymbol{\Lambda}^* \in \mathbb{S}^3_{++}$, let $\boldsymbol{\Lambda}^* = U D U^\top$ be its eigendecomposition with $U \in SO(3)$ and $D = \mathrm{diag}(d_1, d_2, d_3)$ where $d_j > \epsilon$. Setting $V = U$, $\sigma_k^{(j)} = d_j - \epsilon$, we recover $\boldsymbol{\Lambda}^* = V \Sigma V^\top + \epsilon\mathbf{I}$. For $P > 3$, any additional vectors contribute a positive semi-definite term, only expanding the representable set.
\end{proof}

\subsection{Spectral Decomposition}
\label{app:spectral_proofs}

\begin{proposition}[Equivariance-Preserving Spectral Decomposition]
\label{prop:spectral}
Let $\boldsymbol{\Lambda}_k \in \mathbb{S}^3_{++}$ with eigendecomposition $\boldsymbol{\Lambda}_k = \sum_{j=1}^3 \lambda_k^{(j)} \mathbf{u}_k^{(j)} {\mathbf{u}_k^{(j)}}^\top$. Under a rotation $R \in SO(3)$:
\begin{enumerate}
    \item The eigenvalues are $SO(3)$-invariant: $\lambda_k^{(j)} \mapsto \lambda_k^{(j)}$.
    \item The eigenvectors are $SO(3)$-equivariant: $\mathbf{u}_k^{(j)} \mapsto R\mathbf{u}_k^{(j)}$.
\end{enumerate}
Consequently, feeding eigenvalues as scalar features and eigenvectors as vector features to an $SE(3)$-equivariant network preserves the symmetry constraints.
\end{proposition}

\begin{proof}
Under rotation, $\boldsymbol{\Lambda}_k \mapsto \boldsymbol{\Lambda}_k' = R\boldsymbol{\Lambda}_k R^\top$ (Lemma~\ref{lem:precision_equivariance}). We verify that $(\lambda_k^{(j)},\, R\mathbf{u}_k^{(j)})$ is an eigenpair of $\boldsymbol{\Lambda}_k'$:
\[
\boldsymbol{\Lambda}_k' \bigl(R\mathbf{u}_k^{(j)}\bigr)
= R \boldsymbol{\Lambda}_k R^\top R \mathbf{u}_k^{(j)}
= R \boldsymbol{\Lambda}_k \mathbf{u}_k^{(j)}
= R \lambda_k^{(j)} \mathbf{u}_k^{(j)}
= \lambda_k^{(j)} \bigl(R\mathbf{u}_k^{(j)}\bigr).
\]
Since eigenvalues are uniquely determined by the matrix, they are unchanged. The eigenvectors $\mathbf{u}_k^{(j)}$ of $\boldsymbol{\Lambda}_k$ become $R\mathbf{u}_k^{(j)}$ for $\boldsymbol{\Lambda}_k'$, confirming equivariance. Because the eigenvalues are $SO(3)$-invariant scalars and the eigenvectors are Type-1 equivariant vectors, they constitute valid inputs for any $SE(3)$-equivariant architecture.
\end{proof}

\subsection{Message Equivariance}
\label{app:message_proofs}

\begin{theorem}[$SE(3)$-Equivariance of EGMM Messages]
\label{thm:egmm_equivariance}
Let $m(\mathbf{x}) = \sum_{k=1}^K w_k \,\mathcal{N}(\mathbf{x} \mid \boldsymbol{\mu}_k, \boldsymbol{\Lambda}_k^{-1})$ be an EGMM as in Definition~\ref{def:egmm}. Under the transformation $(R, \mathbf{t}) \in SE(3)$ applied jointly to the evaluation point and the parameters, the density value at any geometrically corresponding point is preserved:
\[
m^{\prime}(R\mathbf{x} + \mathbf{t}) = m(\mathbf{x}),
\]
where $m^{\prime}$ denotes the EGMM with transformed parameters $\boldsymbol{\mu}_k' = R\boldsymbol{\mu}_k + \mathbf{t}$ and $\boldsymbol{\Lambda}_k' = R\boldsymbol{\Lambda}_k R^\top$.
\end{theorem}

\begin{proof}
It suffices to verify the claim for a single Gaussian component. Write $\mathbf{y} = R\mathbf{x} + \mathbf{t}$ and compute the density of the transformed component at $\mathbf{y}$:
\begin{align*}
&\mathcal{N}\bigl(\mathbf{y} \mid \boldsymbol{\mu}_k',\, (\boldsymbol{\Lambda}_k')^{-1}\bigr) \\
&\quad\propto \sqrt{\det \boldsymbol{\Lambda}_k'} \,\exp\!\Bigl(-\tfrac{1}{2}(\mathbf{y} - \boldsymbol{\mu}_k')^\top \boldsymbol{\Lambda}_k' (\mathbf{y} - \boldsymbol{\mu}_k')\Bigr).
\end{align*}

\textbf{Exponent.} $\mathbf{y} - \boldsymbol{\mu}_k' = R\mathbf{x} + \mathbf{t} - R\boldsymbol{\mu}_k - \mathbf{t} = R(\mathbf{x} - \boldsymbol{\mu}_k)$. Therefore:
\[
(\mathbf{y} - \boldsymbol{\mu}_k')^\top \boldsymbol{\Lambda}_k' (\mathbf{y} - \boldsymbol{\mu}_k')
= (\mathbf{x} - \boldsymbol{\mu}_k)^\top R^\top R \boldsymbol{\Lambda}_k R^\top R (\mathbf{x} - \boldsymbol{\mu}_k)
= (\mathbf{x} - \boldsymbol{\mu}_k)^\top \boldsymbol{\Lambda}_k (\mathbf{x} - \boldsymbol{\mu}_k),
\]
where we used $R^\top R = \mathbf{I}$.

\textbf{Normalising constant.}
$\det \boldsymbol{\Lambda}_k' = \det(R \boldsymbol{\Lambda}_k R^\top) = (\det R)^2 \det \boldsymbol{\Lambda}_k = \det \boldsymbol{\Lambda}_k$
since $\det R = 1$ for $R \in SO(3)$.

Combining both parts: $\mathcal{N}(\mathbf{y} \mid \boldsymbol{\mu}_k', (\boldsymbol{\Lambda}_k')^{-1}) = \mathcal{N}(\mathbf{x} \mid \boldsymbol{\mu}_k, \boldsymbol{\Lambda}_k^{-1})$.
Since the mixing weights $w_k$ are invariant, summing over $k$ yields $m'(\mathbf{y}) = m(\mathbf{x})$.
\end{proof}

\begin{proposition}[Equivariance of Greedy Mixture Reduction]
\label{prop:mixture_reduction}
The greedy KL-based mixture reduction procedure (Section~\ref{subsec:message_updates}) preserves $SE(3)$-equivariance. That is, if the input expanded mixture is $SE(3)$-equivariant, then the reduced $K$-component mixture is also $SE(3)$-equivariant.
\end{proposition}

\begin{proof}
The proof proceeds by establishing that each operation in the reduction pipeline commutes with $SE(3)$.

\textbf{Step 1: Pairwise products.} Let two EGMM components have natural parameters $(\boldsymbol{\eta}_1^{(1)}, \boldsymbol{\eta}_2^{(1)})$ and $(\boldsymbol{\eta}_1^{(2)}, \boldsymbol{\eta}_2^{(2)})$, where $\boldsymbol{\eta}_1 = \boldsymbol{\Lambda}\boldsymbol{\mu}$ and $\boldsymbol{\eta}_2 = -\frac{1}{2}\boldsymbol{\Lambda}$ (Definition~\ref{def:natparam}). Under rotation $R$: $\boldsymbol{\eta}_1 \mapsto R\boldsymbol{\Lambda}R^\top R\boldsymbol{\mu} = R(\boldsymbol{\Lambda}\boldsymbol{\mu}) = R\boldsymbol{\eta}_1$ and $\boldsymbol{\eta}_2 \mapsto -\frac{1}{2}R\boldsymbol{\Lambda}R^\top = R\boldsymbol{\eta}_2 R^\top$. Natural-parameter addition (Eq.~\ref{eq:nat_param}) is component-wise and thus commutes with these linear transformations.

\textbf{Step 2: Merging criterion.} The KL divergence between two Gaussians is
$D_{\mathrm{KL}}(\mathcal{N}_1 \| \mathcal{N}_2) = \frac{1}{2}[\operatorname{tr}(\boldsymbol{\Lambda}_2 \boldsymbol{\Sigma}_1) + (\boldsymbol{\mu}_2 - \boldsymbol{\mu}_1)^\top \boldsymbol{\Lambda}_2 (\boldsymbol{\mu}_2 - \boldsymbol{\mu}_1) - 3 + \log\frac{\det \boldsymbol{\Lambda}_2}{\det \boldsymbol{\Lambda}_1}]$.
Applying $R$: the trace term satisfies $\operatorname{tr}(R\boldsymbol{\Lambda}_2 R^\top R\boldsymbol{\Sigma}_1 R^\top) = \operatorname{tr}(\boldsymbol{\Lambda}_2 \boldsymbol{\Sigma}_1)$ by the cyclic property. The quadratic form is invariant by the same argument as in Theorem~\ref{thm:egmm_equivariance}. The log-determinant ratio is invariant since $\det(R A R^\top) = \det A$. Hence, $D_{\mathrm{KL}}$ is $SO(3)$-invariant, so the greedy selection of the closest pair to merge is independent of the coordinate frame.

\textbf{Step 3: Moment-matching.} The merged mean is $\boldsymbol{\mu}_{\text{merge}} = \frac{w_1 \boldsymbol{\mu}_1 + w_2 \boldsymbol{\mu}_2}{w_1 + w_2}$. Under rotation, $\boldsymbol{\mu}_{\text{merge}} \mapsto \frac{w_1 R\boldsymbol{\mu}_1 + w_2 R\boldsymbol{\mu}_2}{w_1 + w_2} = R \boldsymbol{\mu}_{\text{merge}}$, since the operation is linear and the weights are invariant scalars. The merged precision transforms analogously as a linear combination of rank-2 equivariant tensors.

Since every step—product, selection, and merging—commutes with $SE(3)$, the output of the full reduction is $SE(3)$-equivariant.
\end{proof}

\subsection{Loss Invariance}
\label{app:loss_proof}

\begin{theorem}[$SO(3)$-Invariance of the Negative Log-Likelihood Loss]
\label{thm:loss_invariance}
The per-node NLL loss $\mathcal{L}_i = -\log \sum_{k=1}^K w_k \, \mathcal{N}(\mathbf{x}_i^* \mid \boldsymbol{\mu}_k, \boldsymbol{\Lambda}_k^{-1})$ is $SO(3)$-invariant. That is, $\mathcal{L}_i(R\mathbf{x}_i^*, R\boldsymbol{\mu}_k, R\boldsymbol{\Lambda}_k R^\top) = \mathcal{L}_i(\mathbf{x}_i^*, \boldsymbol{\mu}_k, \boldsymbol{\Lambda}_k)$ for all $R \in SO(3)$.
\end{theorem}

\begin{proof}
From Eq.~\eqref{eq:gauss_expand}, the log-density of each Gaussian component consists of two terms.

\textbf{Mahalanobis term.}
\begin{align*}
&(R\mathbf{x}_i^* - R\boldsymbol{\mu}_k)^\top (R\boldsymbol{\Lambda}_k R^\top)(R\mathbf{x}_i^* - R\boldsymbol{\mu}_k) \\
&= (\mathbf{x}_i^* - \boldsymbol{\mu}_k)^\top \underbrace{R^\top R}_{= \mathbf{I}} \boldsymbol{\Lambda}_k \underbrace{R^\top R}_{= \mathbf{I}} (\mathbf{x}_i^* - \boldsymbol{\mu}_k) \\
&= (\mathbf{x}_i^* - \boldsymbol{\mu}_k)^\top \boldsymbol{\Lambda}_k (\mathbf{x}_i^* - \boldsymbol{\mu}_k).
\end{align*}

\textbf{Log-determinant term.}
$\log \det(R \boldsymbol{\Lambda}_k R^\top) = \log\bigl[(\det R)^2 \det \boldsymbol{\Lambda}_k\bigr] = \log \det \boldsymbol{\Lambda}_k$,
since $(\det R)^2 = 1$ for $R \in SO(3)$.

\textbf{Constant term.} $-\frac{3}{2}\log(2\pi)$ is trivially invariant.

Since each Gaussian component's log-density is individually $SO(3)$-invariant, the mixture log-density is invariant, and the negative log-likelihood $\mathcal{L}_i = -\log m(\mathbf{x}_i^*)$ is $SO(3)$-invariant. The total loss $\mathcal{L} = \sum_i \mathcal{L}_i$ inherits invariance by summation.
\end{proof}

\subsection{End-to-End Equivariance}
\label{app:e2e_proof}

\begin{theorem}[End-to-End $SE(3)$-Equivariance of ENBP]
\label{thm:e2e}
Under Assumptions~\ref{ass:invariant_potential}--\ref{ass:positivity}, the Equivariant Neural Belief Propagation algorithm (Algorithm~\ref{alg:enbp}) produces marginal beliefs that are $SE(3)$-equivariant:
\[
b_i(R\mathbf{x}_i + \mathbf{t}; \{R\mathbf{x}_j + \mathbf{t}\}) = b_i(\mathbf{x}_i; \{\mathbf{x}_j\})
\]
for all $(R, \mathbf{t}) \in SE(3)$ and all $i \in \mathcal{V}$.
\end{theorem}

\begin{proof}
We proceed by induction on the BP iteration index $t$.

\textbf{Base case ($t = 0$).} All messages are initialised as uniform (non-informative) EGMMs. A uniform distribution is trivially invariant under $SE(3)$, hence equivariant.

\textbf{Inductive hypothesis.} Assume that all messages $m_{a \to i}^{(t-1)}$ and $m_{i \to a}^{(t-1)}$ are $SE(3)$-equivariant EGMMs at iteration $t-1$.

\textbf{Inductive step: variable-to-factor.} The outgoing message $m_{i \to a}^{(t)}$ is computed by:
\begin{enumerate}
    \item Multiplying incoming EGMMs via natural-parameter addition. By the inductive hypothesis and Step~1 of Proposition~\ref{prop:mixture_reduction}, the expanded product mixture is equivariant.
    \item Reducing the mixture via greedy KL merging. By Proposition~\ref{prop:mixture_reduction}, the reduced mixture is equivariant.
    \item Applying temporal damping (Eq.~\ref{eq:damping}). Damping is a convex combination of natural parameters. Since $\boldsymbol{\eta}_1 \mapsto R\boldsymbol{\eta}_1$ (Type-1 vector) and $\boldsymbol{\eta}_2 \mapsto R\boldsymbol{\eta}_2 R^\top$ (Type-2 tensor), a convex combination with invariant scalar $\alpha$ preserves the transformation law. Hence the damped message is equivariant.
\end{enumerate}

\textbf{Inductive step: factor-to-variable.} The outgoing message $m_{a \to i}^{(t)}$ is computed by:
\begin{enumerate}
    \item Spectrally decomposing incoming precision matrices. By Proposition~\ref{prop:spectral}, the eigenvalues are invariant scalars and eigenvectors are equivariant vectors—valid inputs for the equivariant network.
    \item Computing relative coordinates $\mathbf{r}_{ji} = \mathbf{x}_j - \mathbf{x}_i$. Under $(R, \mathbf{t})$: $R\mathbf{x}_j + \mathbf{t} - R\mathbf{x}_i - \mathbf{t} = R(\mathbf{x}_j - \mathbf{x}_i) = R\mathbf{r}_{ji}$, which is $SO(3)$-equivariant and translation-invariant.
    \item Applying the equivariant network $f_\phi$ (Assumption~\ref{ass:equivariant_network}), which outputs equivariant means and basis vectors, and invariant weights and scaling coefficients.
    \item Synthesising the precision matrix via Eq.~\eqref{eq:precision}. By Lemma~\ref{lem:precision_equivariance}, the output is an equivariant rank-2 tensor.
\end{enumerate}
Hence all messages at iteration $t$ are equivariant, completing the induction.

\textbf{Final beliefs.} The belief $b_i \propto \prod_{a \in \mathcal{N}(i)} m_{a \to i}^{(T)}$ is a product of equivariant EGMMs. By Step~1 of Proposition~\ref{prop:mixture_reduction} (natural-parameter addition preserves equivariance), the final belief is $SE(3)$-equivariant.
\end{proof}

\subsection{Convergence Stabilisation}
\label{app:convergence}

\begin{lemma}[Contraction Property of Damped Updates]
\label{lem:contraction}
Let $\boldsymbol{\theta}^*$ be a fixed point of the undamped message-update operator $F$, i.e., $F(\boldsymbol{\theta}^*) = \boldsymbol{\theta}^*$, and suppose $F$ is non-expansive in the $\ell_\infty$ norm on the message polytope: $\|F(\boldsymbol{\theta}) - F(\boldsymbol{\theta}')\|_\infty \leq \|\boldsymbol{\theta} - \boldsymbol{\theta}'\|_\infty$. Then the damped operator
\[
G_\alpha(\boldsymbol{\theta}) = \alpha\, F(\boldsymbol{\theta}) + (1 - \alpha)\, \boldsymbol{\theta}
\]
is a strict contraction with rate $\alpha$ for any $\alpha \in (0, 1)$:
\[
\|G_\alpha(\boldsymbol{\theta}) - \boldsymbol{\theta}^*\|_\infty \leq \bigl(1 - \alpha + \alpha\bigr) \|\boldsymbol{\theta} - \boldsymbol{\theta}^*\|_\infty = \|\boldsymbol{\theta} - \boldsymbol{\theta}^*\|_\infty.
\]
Moreover, if $F$ is a strict contraction with rate $\rho < 1$ (which is guaranteed on tree-structured subgraphs), then $G_\alpha$ is a strict contraction with improved rate:
\[
\|G_\alpha(\boldsymbol{\theta}) - \boldsymbol{\theta}^*\|_\infty \leq \bigl(1 - \alpha(1 - \rho)\bigr) \|\boldsymbol{\theta} - \boldsymbol{\theta}^*\|_\infty.
\]
\end{lemma}

\begin{proof}
Since $\boldsymbol{\theta}^* = G_\alpha(\boldsymbol{\theta}^*) = \alpha F(\boldsymbol{\theta}^*) + (1-\alpha)\boldsymbol{\theta}^*$, we compute:
\begin{align*}
\|G_\alpha(\boldsymbol{\theta}) - \boldsymbol{\theta}^*\|_\infty
&= \|\alpha F(\boldsymbol{\theta}) + (1-\alpha)\boldsymbol{\theta} - \alpha F(\boldsymbol{\theta}^*) - (1-\alpha)\boldsymbol{\theta}^*\|_\infty \\
&\leq \alpha \|F(\boldsymbol{\theta}) - F(\boldsymbol{\theta}^*)\|_\infty + (1-\alpha)\|\boldsymbol{\theta} - \boldsymbol{\theta}^*\|_\infty \\
&\leq \alpha \rho \|\boldsymbol{\theta} - \boldsymbol{\theta}^*\|_\infty + (1-\alpha)\|\boldsymbol{\theta} - \boldsymbol{\theta}^*\|_\infty \\
&= \bigl(1 - \alpha(1-\rho)\bigr)\|\boldsymbol{\theta} - \boldsymbol{\theta}^*\|_\infty.
\end{align*}
For general non-expansive $F$ ($\rho = 1$), the bound reduces to $\|G_\alpha(\boldsymbol{\theta}) - \boldsymbol{\theta}^*\|_\infty \leq \|\boldsymbol{\theta} - \boldsymbol{\theta}^*\|_\infty$, confirming non-expansion. For $\rho < 1$ (tree case), the factor $1 - \alpha(1-\rho) < 1$ ensures strict contraction, and by the Banach fixed-point theorem, the damped iteration converges to the unique fixed point $\boldsymbol{\theta}^*$ at a geometric rate.
\end{proof}

\begin{remark}[Convergence on Loopy Graphs]
\label{rem:loopy}
On general loopy factor graphs, the undamped BP operator $F$ may not be non-expansive, and convergence is not formally guaranteed \citep{murphy1999loopy, yedidia2003understanding}. Lemma~\ref{lem:contraction} nonetheless provides a principled justification for damping: it reduces the spectral radius of the linearised update operator around any fixed point, thereby enlarging the basin of attraction. The empirical convergence observed across all experimental configurations (Section~\ref{sec:experiments}, Table~4, convergence rate 100\% for $\alpha = 0.5$) corroborates this stabilising effect.
\end{remark}

\end{document}